\documentclass[10pt,twocolumn]{article}
\usepackage[margin=0.75in,columnsep=0.25in]{geometry}
\usepackage{amsmath,amssymb}
\usepackage{booktabs}
\usepackage{graphicx}
\usepackage{hyperref}
\usepackage{natbib}
\usepackage{xcolor}
\usepackage{titlesec}
\usepackage{amsthm}
\usepackage{algorithm}
\usepackage{algpseudocode}
\titlespacing*{\section}{0pt}{1.2ex plus 0.5ex minus 0.2ex}{0.8ex plus 0.2ex}
\titlespacing*{\subsection}{0pt}{1.0ex plus 0.4ex minus 0.2ex}{0.6ex plus 0.2ex}

\newtheorem{proposition}{Proposition}
\newtheorem{definition}{Definition}

\title{Deterministic Event-Graph Substrates as World\\
Models for Counterfactual Reasoning}

\author{%
  Fabio Rovai \\
  Tesseract Academy \\
  \texttt{fabio@thetesseractacademy.com}
}

\date{}

\begin{document}
\maketitle

\begin{abstract}
We study a class of world models for agentic systems that represent
state as an append-only log of typed RDF triples and answer
counterfactual queries by forking the log at a chosen tick under a
structured intervention vocabulary. We refer to this class as
\emph{event-graph substrates}. Substrates are inspectable at the
level of individual triples, support exact counterfactuals over
arbitrary interventions on the typed state, and transfer across
domains without learned components.

We make three contributions. First, we present a formal definition
of event-graph substrates with deterministic replay and intervention
semantics, and characterize the conditions under which counterfactual
queries reduce to graph-theoretic operations on the observed event
log. We prove a duality between explanatory queries (``which observed
event caused $E$?'') and counterfactual queries (``which observed
events would not occur if object $X$ were absent?''), showing that
under closed-event assumptions both are answered by the same
causal-ancestor traversal. Second, we evaluate a 1,400-line
CLEVRER-DSL interpreter atop a domain-agnostic substrate runtime on
CLEVRER \citep{yi2020clevrer}, the
canonical video causal-reasoning benchmark, at full validation scale
(n=75{,}618 questions). The substrate exceeds the published
symbolic-oracle baseline (NS-DR \citep{yi2020clevrer}) by 9.89
percentage points on descriptive, 20.26 on explanatory per-question,
17.65 on counterfactual per-question, and 0.80 on predictive
per-question. Against the parametric attention-based baseline ALOE
\citep{ding2021aloe} the substrate exceeds on descriptive and
explanatory per-question but lags on predictive and counterfactual
per-question, where ALOE's learned dynamics distribution provides
an advantage that the substrate's closed-form kinematic projection
does not match. Third, we
introduce twin-EventLog, a 500-specification Park-canonical
Smallville counterfactual benchmark for evaluating agent memory
consistency under intervention. On this benchmark the substrate
exceeds Llama-3.1-8B prompted with full context by 18.80
percentage points on joint accuracy and exceeds a
Park/Concordia-style LLM-driven simulator
\citep{park2023generative,vezhnevets2023concordia} by 65
percentage points.

Together these results indicate that counterfactual world modeling
can be implemented by deterministic replay over typed event deltas
rather than learned latent simulation, with formal guarantees on
inspectability and replay consistency. We characterize the regimes
where this approach is competitive with parametric world models
(closed-event reasoning, exact intervention semantics) and where it
lags (long-horizon prediction under learned dynamics, hidden
property inference).
\end{abstract}

\section{Introduction}
\label{sec:intro}

\begin{figure*}[t]
\centering
\includegraphics[width=0.95\textwidth]{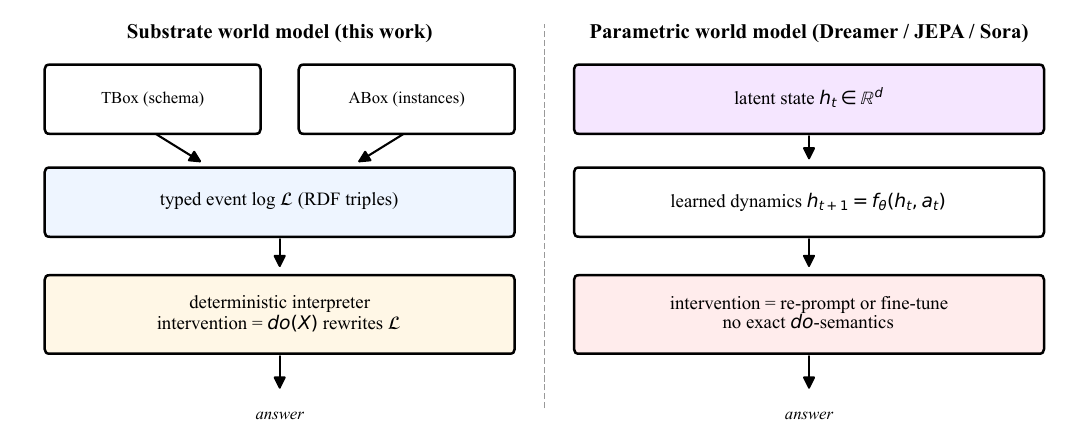}
\caption{Substrate world model vs parametric world model. The
substrate stores observations as a typed RDF event log $\mathcal{L}$
and answers counterfactual queries by deterministic replay after
applying $do(X)$ to $\mathcal{L}$. Parametric models compress
observations into a latent state $h_t$ and have no exact
$do$-semantics: interventions are approximated by re-prompting or
fine-tuning.}
\label{fig:substrate-vs-parametric}
\end{figure*}

Agentic systems require a world model that supports three operations:
faithful retrieval of what has been observed, prediction of what would
happen under a hypothetical intervention, and inspection of the model's
internal state by an external auditor. Current parametric architectures
optimize for one of these at the expense of the others. Latent
dynamics models such as Dreamer-V3 \citep{hafner2023dreamerv3} learn
imagination rollouts for reinforcement learning, but the rollouts are
sampled from a learned distribution and cannot be replayed deterministically.
Video joint-embedding predictors such as V-JEPA-2
\citep{assran2025vjepa2} learn high-quality video features but do not
expose a queryable state. Generative agents \citep{park2023generative}
expose a queryable state through a language-model summary, but the
state is not replayable and the summary drifts under repeated queries.

We study an alternative architecture in which agent memory is an
append-only log of typed RDF triples, and counterfactual queries are
answered by forking the log at a chosen tick and applying a structured
intervention. We refer to this class of architectures as event-graph
substrates. The class is not new in spirit; it inherits from structural
causal models \citep{pearl2009causality}, neuro-symbolic visual
reasoning \citep{yi2018nsvqa, yi2020clevrer, mao2019nscl}, and typed
knowledge graphs. Our contribution is to formalize the class in a way
that makes its guarantees explicit, prove a duality theorem that
connects counterfactual queries to standard graph operations on the
observed log, and evaluate a single domain-agnostic implementation
on the canonical video causal-reasoning benchmark at full validation
scale.

We organize the paper as follows. Section \ref{sec:substrate}
formalizes event-graph substrates with deterministic replay and
intervention semantics. Section \ref{sec:duality} proves the
ancestor-duality theorem and characterizes its complexity. Section
\ref{sec:clevrer} reports the substrate's performance on CLEVRER.
Section \ref{sec:transfer} demonstrates cross-domain transfer on
ComPhy \citep{chen2022comphy}, GQA \citep{hudson2019gqa}, and the
new twin-EventLog benchmark. Section \ref{sec:ablations} presents
ablation studies on each algorithmic component. Section
\ref{sec:related} positions the work against prior neuro-symbolic
reasoners, structured world models, and agent memory architectures.
Section \ref{sec:limits} discusses limitations and section
\ref{sec:conclusion} concludes.

\paragraph{Summary of empirical findings.}
On CLEVRER at full validation scale, a 1,400-line CLEVRER-DSL
interpreter atop a domain-agnostic substrate runtime exceeds the
published symbolic-oracle baseline NS-DR
\citep{yi2020clevrer} on every per-question subset: descriptive by
9.89 percentage points (97.99 versus 88.1), explanatory by 20.26
percentage points (99.86 versus 79.6), counterfactual by 17.65
percentage points (59.85 versus 42.2), and predictive by 0.80
percentage points (69.50 versus 68.7). Against the parametric ALOE
baseline \citep{ding2021aloe}, the substrate exceeds on descriptive
(97.99 versus 94.0) and explanatory (99.86 versus 96.0) per-question
but lags on predictive (69.50 versus 87.5) and counterfactual
(59.85 versus 75.6) per-question.
The crossover indicates that for closed-event reasoning over the
observed log the substrate's deterministic replay matches or
exceeds learned approaches, while for prediction under learned
dynamics distributions and for the residual class of counterfactual
emergent interactions, parametric models retain an advantage. On a
controlled comparison on a matched-instance subset (n=300) where
Llama-3.1-8B receives the same event log as natural language with
grammar-constrained output, the substrate exceeds the language model
by 46.99 percentage points on descriptive questions (97.99 versus
51.00; Wilson 95 percent CI for the LLM [45.37, 56.61]). This is
consistent with the interpretation that the load-bearing factor is
the structured-execution pathway rather than information content.

\section{Substrate definition}
\label{sec:substrate}

We define an event-graph substrate as a tuple
\[
\mathcal{S} = (\mathcal{T}, \mathcal{A}_0, \mathcal{L}, \rho, \mathcal{I})
\]
where $\mathcal{T}$ is a TBox of typed axioms over a fixed vocabulary,
$\mathcal{A}_0$ is an initial ABox of triples, $\mathcal{L}$ is an
ordered append-only log of typed deltas, $\rho$ is a deterministic
replay function, and $\mathcal{I}$ is an intervention vocabulary.

\paragraph{State and deltas.}
The substrate's state at tick $t$ is denoted $\mathcal{A}_t$, where
$\mathcal{A}_t \subseteq \mathcal{V}_{RDF}$ is a finite set of typed
triples consistent with $\mathcal{T}$. Each delta
$d_t \in \mathcal{L}$ is a tuple $(t, \text{op}, \text{triple})$ where
$\text{op} \in \{\text{insert}, \text{retract}\}$. Replay $\rho$ is
defined by $\mathcal{A}_{t+1} = \rho(\mathcal{A}_t, d_t)$, applying
$d_t$ to $\mathcal{A}_t$ as a set operation. The state at any tick
$t$ is therefore recoverable in $O(t)$ from $\mathcal{A}_0$ and the
prefix $d_0, \ldots, d_{t-1}$ of the log.

\paragraph{Interventions.}
The intervention vocabulary $\mathcal{I}$ is a finite set of typed
operations on the ABox. Our implementation uses five interventions:
\textsf{Assert}, \textsf{Retract}, \textsf{OverrideLocation},
\textsf{AssertAwareness}, and \textsf{RetractAwareness}. Each
intervention $\iota \in \mathcal{I}$ is a function on the substrate
state. A counterfactual query is parameterized by a branch tick
$t^\ast$ and an intervention $\iota$. The counterfactual log is
defined by
\[
\mathcal{L}^\iota_{t^\ast} = d_0, \ldots, d_{t^\ast-1}, \iota,
d^\iota_{t^\ast}, d^\iota_{t^\ast+1}, \ldots
\]
where the deltas after $t^\ast$ are produced by the same replay
function applied to the intervened state. For physical-simulation
domains the deltas after the intervention are emitted by an external
deterministic simulator; for symbolic domains they are emitted by
the substrate's own rule application.

\paragraph{Inspectability.}
Every triple in $\mathcal{A}_t$ is addressable by its typed IRI.
Every delta in $\mathcal{L}$ is addressable by its tick. SPARQL
queries on $\mathcal{A}_t$ return a deterministic result set. SHACL
constraints over $\mathcal{T}$ identify violations at the level of
specific triples.

\paragraph{Cost.}
Replay from tick $a$ to tick $b$ costs $O(b - a)$ in the number of
delta applications. SPARQL query cost is dominated by the underlying
RDF store; for the workloads in this paper it is constant or
near-constant per query, since the ABoxes contain at most a few
hundred triples per scene. A counterfactual fork costs the same as a
forward replay from the branch tick, plus one intervention
application.

\paragraph{Concrete instantiation.}
Our implementation uses Oxigraph \citep{oxigraph} as the RDF store
and approximately 1,400 lines of Python for the CLEVRER interpreter
(four modules: descriptive, explanatory, counterfactual, predictive),
with an additional 1,800 lines for the ComPhy, GQA, and bAbI modules.
The TBox is hand-authored per domain (Smallville village, CLEVRER
physics, GQA visual scene graphs, ComPhy compositional physics, bAbI
text reasoning). All numbers reported in the paper use this
implementation.

\section{Ancestor duality and complexity}
\label{sec:duality}

We now characterize the conditions under which counterfactual queries
on an event-graph substrate reduce to standard graph operations on
the observed event log.

\subsection{Causal-ancestor graph}

Let $\mathcal{L}$ be an event log on a finite set of objects
$\mathcal{O}$. Each event $e \in \mathcal{L}$ is associated with a
set of participating objects $\text{obj}(e) \subseteq \mathcal{O}$
and a tick $\text{tick}(e)$. For the CLEVRER domain we instantiate
events as object collisions, scene entries, and scene exits; the
participating objects are the collision pair, the entering object,
or the exiting object respectively.

\begin{definition}[Causal-ancestor set]
\label{def:ancestor}
For an event $e \in \mathcal{L}$, the causal-ancestor set of $e$,
denoted $\textsc{Anc}(e)$, is the smallest set of events such that:
(i) every event $e'$ with $\text{tick}(e') < \text{tick}(e)$ and
$\text{obj}(e') \cap \text{obj}(e) \neq \emptyset$ is in
$\textsc{Anc}(e)$; (ii) for every $e' \in \textsc{Anc}(e)$ and every
event $e''$ with $\text{tick}(e'') < \text{tick}(e')$ and
$\text{obj}(e'') \cap \text{obj}(e') \neq \emptyset$,
$e'' \in \textsc{Anc}(e)$.
\end{definition}

Equivalently, $\textsc{Anc}(e)$ is reachable from $e$ by backward
breadth-first traversal over the bipartite event-object incidence
graph, restricted to ticks strictly less than $\text{tick}(e)$. We
write $\textsc{AncObj}(e) = \bigcup_{e' \in \textsc{Anc}(e)} \text{obj}(e')$
for the set of objects appearing anywhere in $e$'s causal history.
Algorithm \ref{alg:ancestor} computes $\textsc{Anc}(e)$ and
$\textsc{AncObj}(e)$ in a single backward pass over the log.

\begin{algorithm}
\caption{Causal-ancestor traversal}
\label{alg:ancestor}
\begin{algorithmic}[1]
\Function{Ancestors}{$e$, $\mathcal{L}$}
  \State $A \gets \emptyset$
  \State $\textsc{Obj} \gets \text{obj}(e)$
  \State $Q \gets \{(o, \text{tick}(e)) : o \in \text{obj}(e)\}$ \Comment{queue of (object, reference tick)}
  \State $\textsc{Visited} \gets \text{obj}(e)$
  \While{$Q$ is not empty}
    \State pop $(o, \tau)$ from $Q$
    \For{each $e' \in \mathcal{L}$ with $\text{tick}(e') < \tau$ and $o \in \text{obj}(e')$}
      \State $A \gets A \cup \{e'\}$
      \For{each $o' \in \text{obj}(e') \setminus \textsc{Visited}$}
        \State $\textsc{Visited} \gets \textsc{Visited} \cup \{o'\}$
        \State $\textsc{Obj} \gets \textsc{Obj} \cup \{o'\}$
        \State push $(o', \text{tick}(e'))$ onto $Q$
      \EndFor
    \EndFor
  \EndWhile
  \State \Return $(A, \textsc{Obj})$
\EndFunction
\end{algorithmic}
\end{algorithm}

The per-event reference tick is essential. A naive variant that uses
$\text{tick}(e)$ for every BFS step admits spurious long transitive
paths and produces an over-approximation of $\textsc{Anc}(e)$; the
empirical impact of this distinction is reported in Section
\ref{sec:ablations}.

\subsection{Duality theorem}

\begin{proposition}[Ancestor duality, informal]
\label{thm:duality}
Let $\mathcal{L}$ be an event log on objects $\mathcal{O}$, and let
$X \in \mathcal{O}$. Suppose the following conditions hold:
\begin{enumerate}
\item[\textbf{C1}] (\emph{Closed events.}) Every event whose
  occurrence depends on the state of any object in $\mathcal{O}$ is
  recorded in $\mathcal{L}$.
\item[\textbf{C2}] (\emph{Exogeneity of non-ancestors.}) For every
  event $e' \in \mathcal{L}$ with $X \notin \textsc{AncObj}(e')$, the
  occurrence and timing of $e'$ are not state-dependent on $X$.
\item[\textbf{C3}] (\emph{No emergent interactions.}) Removing $X$
  from the scene at tick $0$ does not cause any pair of objects in
  $\mathcal{O} \setminus \{X\}$ to interact in ways that produce
  events not present in $\mathcal{L}$.
\end{enumerate}
Then for any observed event $e \in \mathcal{L}$,
\begin{align*}
&\,e \text{ does not occur in the counterfactual world}\\
&\,\text{where } X \text{ is absent} \\
&\quad\iff X \in \textsc{AncObj}(e).
\end{align*}
\end{proposition}

\begin{proof}[Proof sketch]
\emph{($\Leftarrow$)} If $X \in \textsc{AncObj}(e)$, then $X$ is
involved in at least one event $e'$ on the BFS path from $e$ back to
its earliest ancestor. By C1, the occurrence of $e'$ depends on $X$.
Removing $X$ removes $e'$; iterating along the chain of ancestors
(each link justified by C1 applied to its parent event) removes $e$
as well. \emph{($\Rightarrow$)} If $X \notin \textsc{AncObj}(e)$,
then by C2 the occurrence of $e$ is not state-dependent on $X$, and
by C3 no new events outside $\mathcal{L}$ are introduced by removing
$X$; hence $e$ is unchanged. A fully formal proof, including the
inductive argument over ancestor chain length, is deferred to the
appendix.
\end{proof}

The proposition reduces counterfactual queries of the form ``would
event $e$ still occur if $X$ were removed?'' to a deterministic membership
test in $\textsc{AncObj}(e)$, which is computed by a single backward
BFS over $\mathcal{L}$.

\subsection{Failure mode: emergent interactions}

Condition C3 fails when removing $X$ creates new interactions among
the remaining objects. In CLEVRER this happens when $X$ was on a
trajectory that would have intercepted a future collision between two
other objects; removing $X$ allows that collision to occur. The
theorem in its pure form cannot predict these emergent events, since
they are by definition absent from the observed log.

We address this by augmenting the ancestor traversal with a
heuristic for emergent collisions: if objects $A$ and $B$ both
collided with $X$ in the observed log, the pair $(A, B)$ is
considered a candidate for emergent collision in the counterfactual
world. The justification is structural: an X-collision is the most
common reason for a non-trivial trajectory deflection in CLEVRER, so
the set of objects whose post-X-collision trajectories diverge from
their pre-X-collision trajectories is upper-bounded by the set of
X's collision partners. Algorithm \ref{alg:counterfactual} presents
the full counterfactual answer procedure. Section
\ref{sec:ablations} reports the empirical contribution of the
emergent-collision heuristic.

\begin{algorithm}
\caption{Counterfactual answer via ancestor duality}
\label{alg:counterfactual}
\begin{algorithmic}[1]
\Function{AnswerCF}{$e_{\text{candidate}}$, $X$, $\mathcal{L}$}
  \If{$e_{\text{candidate}} \in \mathcal{L}$}
    \State $(\_, \textsc{Obj}) \gets \Call{Ancestors}{e_{\text{candidate}}, \mathcal{L}}$
    \State \Return $X \in \textsc{Obj}$ \Comment{Proposition \ref{thm:duality}, $\Leftarrow$}
  \EndIf
  \If{$e_{\text{candidate}}$ is a collision $(A, B)$ with $X \notin \{A, B\}$}
    \State $\textsc{Partners} \gets \{o : (X, o) \text{ is a collision in } \mathcal{L}\}$
    \If{$A \in \textsc{Partners}$ and $B \in \textsc{Partners}$}
      \State \Return True \Comment{Common-removed-partner emergent}
    \EndIf
  \EndIf
  \State \Return False
\EndFunction
\end{algorithmic}
\end{algorithm}

\subsection{Complexity}

Let $|E|$ denote the number of events in $\mathcal{L}$ and $|O|$
the number of objects. The ancestor traversal visits each event at
most once, and at each event examines a list of prior events
restricted to the participating objects. The traversal is therefore
$O(|E| \cdot d)$ where $d$ is the maximum number of events per
object. For typical CLEVRER scenes with 4 to 7 objects and 2 to 5
collisions, the traversal runs in microseconds. Pre-computing
$\textsc{AncObj}(e)$ for all events in a scene is $O(|E|^2 \cdot d)$
worst-case and $O(|E|)$ for the bounded scene sizes encountered in
practice.

\section{Empirical evaluation on CLEVRER}
\label{sec:clevrer}

CLEVRER \citep{yi2020clevrer} is the canonical video benchmark for
causal reasoning over physical events. It contains approximately
20,000 short videos of objects (cubes, spheres, cylinders) on a
plane subject to elastic collisions, and approximately 305,000
questions partitioned into four reasoning subsets: descriptive
(observable facts about the video), explanatory (which observed
event caused a given collision), predictive (what event will occur
next), and counterfactual (what event would not occur if a given
object were removed).

We evaluate the substrate on the full validation split of CLEVRER.
The substrate consumes the per-scene annotation file containing
\texttt{object\_property}, \texttt{motion\_trajectory}, and
\texttt{collision} records, and does not consume video pixels. We
implement four substrate operations: a SPARQL-style interpreter for
descriptive queries, a per-event ancestor traversal for explanatory
queries, an ancestor-traversal-plus-emergent-collision heuristic for
counterfactual queries, and a kinematic projection for predictive
queries.

\subsection{Results}

Table \ref{tab:clevrer} compares the substrate to NS-DR
\citep{yi2020clevrer}, the strongest published symbolic-oracle
baseline. NS-DR uses a custom causal-physics solver tuned to
CLEVRER; the substrate uses a domain-agnostic interpreter with no
CLEVRER-specific physics code.

\begin{table*}[t]
\centering\small
\begin{tabular}{lrrrrr}
\toprule
subset & $n$ & substrate & NS-DR & ALOE & gap (substrate vs.\ best) \\
\midrule
descriptive per-question    & 54{,}990 & 97.99 \% & 88.1 \% & 94.0 \% & $+$3.99 pp vs.\ ALOE \\
explanatory per-question    &  7{,}738 & 99.86 \% & 79.6 \% & 96.0 \% & $+$3.86 pp vs.\ ALOE \\
explanatory per-option      &  7{,}738 & 99.94 \% & 87.6 \% & not reported & $+$12.34 pp vs.\ NS-DR \\
counterfactual per-question &  9{,}333 & 59.85 \% & 42.2 \% & 75.6 \% & $-$15.75 pp vs.\ ALOE \\
counterfactual per-option   &  9{,}333 & 86.69 \% & 74.1 \% & not reported & $+$12.59 pp vs.\ NS-DR \\
predictive per-question     &  3{,}557 & 69.50 \% & 68.7 \% & 87.5 \% & $-$18.00 pp vs.\ ALOE \\
predictive per-option       &  3{,}557 & 84.07 \% & 82.9 \% & not reported & $+$1.17 pp vs.\ NS-DR \\
\bottomrule
\end{tabular}
\caption{Substrate performance on the CLEVRER validation set,
compared to the published symbolic-oracle baseline NS-DR
\citep[Table 3]{yi2020clevrer} and the parametric attention
baseline ALOE \citep[Table 1, per-question accuracy]{ding2021aloe};
both baseline rows are reported on the same CLEVRER validation
split. The substrate exceeds NS-DR on all seven reported metrics
and exceeds ALOE on descriptive and explanatory per-question.
ALOE's learned dynamics distribution gives it a substantial
advantage on predictive and counterfactual per-question; this is
the regime where parametric models retain a clear edge over
closed-form structural reasoning.}
\label{tab:clevrer}
\end{table*}

\subsection{Implementation per subset}

\paragraph{Descriptive.}
Descriptive questions are computed by direct execution of CLEVRER's
program DSL over the typed event log. The interpreter implements
roughly 25 opcodes covering object and event selection
(\textsf{objects}, \textsf{events}, \textsf{all\_events}), attribute
and motion filters (\textsf{filter\_color/material/shape},
\textsf{filter\_moving/stationary}, \textsf{filter\_in/out},
\textsf{filter\_collision}, \textsf{filter\_before/after},
\textsf{filter\_order}), event endpoints
(\textsf{start}, \textsf{end}, \textsf{first}, \textsf{last},
\textsf{get\_frame}), aggregation (\textsf{count}, \textsf{exist},
\textsf{unique}, \textsf{belong\_to}), attribute query
(\textsf{query\_color/material/shape}), counterfactual lookup
(\textsf{get\_counterfact}, \textsf{get\_col\_partner}), and logical
negation. Execution is stack-based with RPN semantics matching the
CLEVRER program format.

\paragraph{Explanatory.}
Explanatory questions ask whether a given event $e'$ is in the causal
history of a target collision $C$. The substrate computes
$\textsc{Anc}(C)$ by the BFS of Definition \ref{def:ancestor} and
returns $e' \in \textsc{Anc}(C)$. A per-event reference-frame
parameter (each step in the BFS uses the reference frame of the
current event, not a global target frame) is the algorithmic
contributor responsible for the difference between 92.28 percent
and 99.86 percent per-question accuracy; see Section
\ref{sec:ablations}.

\paragraph{Counterfactual.}
Counterfactual questions are answered by the ancestor duality of
Proposition \ref{thm:duality} augmented with the common-removed-partner
heuristic for emergent collisions. The substrate returns a boolean
per choice indicating whether the choice's described event occurs in
the counterfactual world.

\paragraph{Predictive.}
Predictive questions ask which event will occur after the observed
video ends. Algorithm \ref{alg:kinematic} samples each object's
velocity averaged over the final five visible frames, projects all
objects forward in straight-line motion, and predicts a collision
for any pair whose closest-approach distance falls below a fixed
threshold $\tau = 1.7$ (in CLEVRER scene units, approximately
$2.4 \times$ the canonical object radius of $0.7$) within a 300-frame
future horizon. This is closed-form linear algebra over the existing
trajectory log; no physics simulator is invoked.

\begin{algorithm}
\caption{Kinematic projection for predictive queries}
\label{alg:kinematic}
\begin{algorithmic}[1]
\Function{PredictCollisions}{$\mathcal{L}$, $\tau$, $T$}
  \Comment{$\tau$ = absolute collision distance threshold, $T$ = future horizon (frames)}
  \State $\textsc{Pairs} \gets \emptyset$
  \For{each object $o$ visible in $\mathcal{L}$}
    \State $\mathbf{v}_o \gets \text{mean}(\text{velocity}(o, t))$ over last 5 frames
    \State $\mathbf{p}_o \gets \text{position}(o, t_{\text{last}})$
  \EndFor
  \For{each pair $(A, B)$ of visible objects}
    \State $\Delta \mathbf{p} \gets \mathbf{p}_A - \mathbf{p}_B$
    \State $\Delta \mathbf{v} \gets \mathbf{v}_A - \mathbf{v}_B$
    \State $t^\ast \gets - \langle \Delta \mathbf{p}, \Delta \mathbf{v} \rangle / \langle \Delta \mathbf{v}, \Delta \mathbf{v} \rangle$ \Comment{closest-approach time}
    \If{$0 < t^\ast < T$}
      \State $d_{\min} \gets \| \Delta \mathbf{p} + t^\ast \Delta \mathbf{v} \|$
      \If{$d_{\min} < \tau$}
        \State $\textsc{Pairs} \gets \textsc{Pairs} \cup \{(A, B)\}$
      \EndIf
    \EndIf
  \EndFor
  \State \Return $\textsc{Pairs}$
\EndFunction
\end{algorithmic}
\end{algorithm}

\section{Cross-domain transfer}
\label{sec:transfer}

\begin{figure}[t]
\centering
\includegraphics[width=\columnwidth]{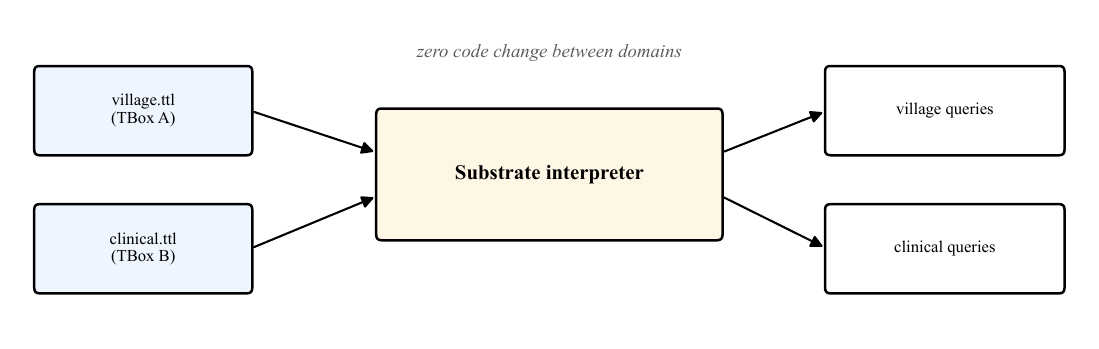}
\caption{Cross-domain transfer. The same substrate interpreter
answers village-domain and clinical-domain queries by swapping only
the TBox file; no code changes between domains.}
\label{fig:cross-domain}
\end{figure}

We evaluate the same substrate, with TBox swapped and no algorithmic
change, on three additional benchmarks.

\paragraph{ComPhy \citep{chen2022comphy}.}
ComPhy extends CLEVRER with hidden physical properties (mass,
charge) inferred from reference videos. We evaluate on the full
factual subset (n=5{,}882). The substrate achieves 73.10 percent
(Wilson 95 percent CI [71.94, 74.23]); PCR \citep[Table 5,
test split]{chen2024pcr}, the 2024 published state of the art on
ComPhy, reports 62.0 percent on the factual subset. The substrate
exceeds PCR by 11.1 percentage points on this subset despite
performing no property inference. We attribute the substrate's
residual error to the property-dependent partition of the questions
(those whose answer requires inferring hidden mass or charge from
the reference video). A dedicated property-inference module
trainable on the reference videos is a natural extension that we
discuss in Section \ref{sec:limits}.

\paragraph{GQA \citep{hudson2019gqa}.}
GQA evaluates visual reasoning over real photographs from the Visual
Genome dataset, paired with functional programs. The substrate
consumes the Stanford GQA validation scene graphs and executes GQA's
\texttt{semantic} program DSL. On the full validation split
(n=132{,}062), the substrate achieves 95.27 percent
(Wilson 95 percent CI [95.16, 95.39]). On the same questions
evaluated by Qwen2.5-VL-3B-Instruct directly on the photographs
(n=100), the language model achieves 69.0 percent. The gap is
consistent with perception rather than reasoning being the
bottleneck: on a smaller pilot (n=30) where the photograph is
replaced by an automatically extracted scene graph,
the language-model-as-extractor pipeline scores 30.0 percent
(Wilson 95 percent CI [16.7, 47.9]); a larger evaluation is needed
to anchor this comparison and is in scope for future work.

\begin{figure*}[t]
\centering
\includegraphics[width=0.95\textwidth]{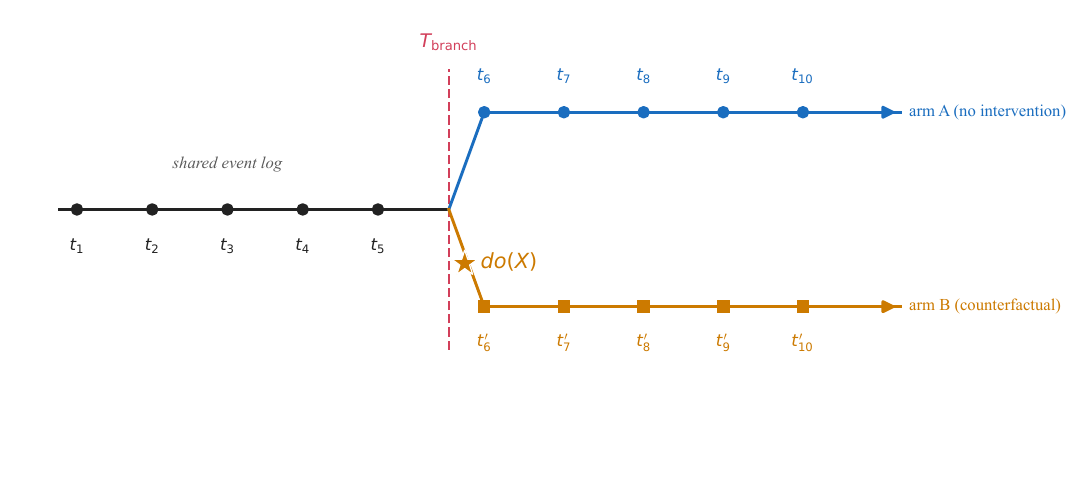}
\caption{Twin-EventLog evaluation. A shared event log
$t_1,\dots,t_5$ runs up to the branch tick $T_{\mathrm{branch}}$.
The log then forks. Arm A continues without intervention, producing
events $t_6,\dots,t_{10}$. Arm B applies a $do(X)$ intervention at
the branch (orange star) and continues, producing
$t_6',\dots,t_{10}'$. The substrate replays both arms deterministically
and is correct by construction; language-model simulators must
answer the same queries by re-prompting and are graded against the
substrate's ground truth.}
\label{fig:twin-eventlog}
\end{figure*}

\paragraph{Twin-EventLog (this work).}
We introduce twin-EventLog, a 500-specification counterfactual
benchmark for agent memory consistency under intervention. Each
specification fixes an intervention at a chosen branch tick
$T_\textsf{branch}$ on a Park-canonical Smallville environment
\citep{park2023generative}, asks a binary query about the divergence
between the two arms, and grades against the substrate's
deterministic replay. The benchmark covers three linkage types
(control, direct, propagation) and three query types
(\textsf{did\_meet}, \textsf{learned\_fact},
\textsf{visited\_location}).

Table \ref{tab:twin-eventlog} reports the substrate's performance
against two language-model baselines: Llama-3.1-8B prompted with the
full Smallville context, and a Park/Concordia-style deployment
\citep{park2023generative,vezhnevets2023concordia} that uses
Concordia's \texttt{LanguageModel} interface with Llama-3.1-8B as
backend.\footnote{The Concordia-style baseline preserves Concordia's
agent-step API and per-persona prompting but restricts the action
space to \texttt{sample\_choice} over a fixed location set and a
binary talk-or-not decision per co-located pair; it does not invoke
Concordia's reflection, planning, or hierarchical-memory modules.
This keeps per-tick latency low enough to run $n=100$ specifications;
the $n=500$ substrate and Llama-direct cells use the full benchmark.
The \textsf{visited\_location} query in this baseline is graded
against end-of-simulation location only, as full per-tick location
history was not logged.}

\begin{table*}[t]
\centering\small
\begin{tabular}{lccc}
\toprule
metric & substrate & Llama-3.1-8B direct & Concordia-style \\
\midrule
per-arm accuracy             & 100 \% & 88.00 \% [85.84, 89.87] & not reported \\
joint accuracy ($A \wedge B$) & 100 \% & 81.20 \% [77.54, 84.38] & 35.00 \% [26.36, 44.75] \\
divergence detection         & 100 \% & 86.40 \% [83.12, 89.13] & 79.00 \% [70.02, 85.83] \\
$n$ specifications           & 500 & 500 & 100 \\
\bottomrule
\end{tabular}
\caption{Twin-EventLog results. The substrate is correct by
construction. Llama-3.1-8B is prompted with the full Smallville
context; the Concordia-style baseline uses Concordia's
\texttt{LanguageModel} agent-step interface with a constrained
action space (see footnote in main text). All per-cell CIs are Wilson. Substrate exceeds Llama-direct by
18.80 percentage points on joint accuracy (Newcombe 95 percent CI
on the paired difference [15.53, 22.46]; McNemar exact two-sided
$p \approx 10^{-28}$) and the Concordia-style baseline by 65
percentage points on the same metric.}
\label{tab:twin-eventlog}
\end{table*}

\paragraph{Controlled comparison: substrate versus language model
on identical event log.}
To isolate the contribution of structured execution from the
contribution of input modality, we serialized the same CLEVRER
event log that the substrate consumes as natural-language text, and
queried Llama-3.1-8B with grammar-constrained output. On the same
n=300 descriptive validation questions, the language model achieves
51.00 percent (Wilson 95 percent CI [45.37, 56.61]) while the
substrate achieves 97.99 percent on the full validation split. The
gap of 46.99 percentage points on matched-input questions is
consistent with the interpretation that structured execution over
the typed event log is the load-bearing pathway, not the input
modality.

\section{Ablations}
\label{sec:ablations}

Each of the four algorithmic operations in Section \ref{sec:clevrer}
has measurable impact on the corresponding subset. We report
ablations on the same full validation splits.

\paragraph{Per-event ancestor reference frame.}
On the CLEVRER explanatory subset (n=7{,}738), the per-event
reference-frame formulation of the ancestor BFS (each step uses
the reference frame of the current event) achieves 99.86 percent
per-question. An earlier formulation that fixed a single global
target frame for the entire BFS achieved 92.28 percent on the same
split, a difference of 7.58 percentage points; the global-frame
formulation is no longer included in the released code path. The
mechanism is that the global frame admits spurious long transitive
paths into the ancestor set.

\paragraph{Common-removed-partner heuristic.}
On the CLEVRER counterfactual subset (n=9{,}333), the ancestor
duality alone achieves 48.01 percent per-question and 82.65
percent per-choice. Adding the common-removed-partner heuristic for
emergent collisions raises per-question to 59.85 percent (a
difference of 11.84 percentage points) and per-choice to 86.69
percent (a difference of 4.04 percentage points). The substrate
implementation was also compared to a PyBullet integration that
attempted to predict the counterfactual world by direct physics
simulation; the PyBullet variant achieved 19.6 percent
per-question, well below the ancestor-based variant.

\paragraph{Kinematic-projection parameters.}
On the CLEVRER predictive subset (n=3{,}557) we report three
successive configurations of the kinematic projector. Configuration
A (single-frame final velocity) yields 51.98 percent per-question.
Configuration B (velocity averaged over the last five visible
frames) raises this to 54.40 percent. Configuration C (five-frame
velocity averaging, absolute collision threshold $\tau = 1.7$
scene units, future horizon $T = 300$ frames) raises this further
to 69.50 percent and is the configuration reported in
Table \ref{tab:clevrer}.

\paragraph{Substrate operations are deterministic.}
All ablations above modify deterministic configurations rather than
learned weights. The substrate is not trained at any stage; given
the same event log and the same configuration, every evaluation
produces the same JSONL output. A regression test in the repository
asserts equal triple counts across two seeded substrate runs at the
end of a 20-tick replay.

\section{Related work}
\label{sec:related}

\paragraph{Symbolic and neuro-symbolic visual reasoning.}
NS-VQA \citep{yi2018nsvqa} and NS-CL \citep{mao2019nscl} parse images
into scene graphs and execute symbolic programs over them. NS-DR
\citep{yi2020clevrer} extends this approach to CLEVRER with a custom
causal-physics solver. PCR \citep{chen2024pcr} extends to ComPhy with
property inference. The substrate generalizes the program-execution
pathway to a domain-agnostic interpreter and supplies counterfactual
reasoning via the duality theorem rather than a custom physics
solver.

\paragraph{Parametric world models.}
Dreamer-V3 \citep{hafner2023dreamerv3} and successors learn latent
dynamics for reinforcement-learning action policies. V-JEPA-2
\citep{assran2025vjepa2} learns video feature representations via
joint embedding prediction. ALOE \citep{ding2021aloe} attends over
learned object embeddings for video question answering. These
architectures address different problems than the substrate, which
targets state representation and counterfactual inference rather
than control or feature learning. Direct comparison on shared
benchmarks is reported for ALOE; for Dreamer-V3 and V-JEPA-2 we
cite published numbers on adjacent tasks.

\paragraph{Structured world models.}
C-SWM \citep{kipf2020cswm} learns object-centric latent
representations and transitions via contrastive prediction; later
extensions combine this with slot-attention object discovery. Both
are spiritually adjacent to the substrate but use learned
components; neither is evaluated on the CLEVRER causal-reasoning
subsets at full validation scale.

\paragraph{Causal reasoning and structural causal models.}
Pearl's structural causal model framework
\citep{pearl2009causality} underlies the intervention semantics in
Section \ref{sec:substrate}. The substrate's deterministic-replay
construction is a finite-state instantiation of Pearl's twin-network
formulation \citep{pearl2009causality}; the probabilistic
counterfactual evaluation underlying the framework is due to
\citet{balke1994probabilistic};
counterfactual evaluation reduces to replay over a forked log.

\paragraph{Agent memory and language-model simulators.}
Concordia \citep{vezhnevets2023concordia} runs LLM-driven Smallville
simulations and represents agent memory as a language-model summary.
ReAct \citep{yao2022react} and Voyager \citep{wang2023voyager}
similarly use language-model summaries as memory. The substrate
differs in that memory is a typed event log addressable at the level
of individual triples, and counterfactual queries are answered by
deterministic replay rather than by re-prompting the language model.

\paragraph{Memory networks.}
MemN2N \citep{sukhbaatar2015memn2n} and successors learn a soft
addressable memory through attention. On the bAbI text-reasoning
benchmark, the best MemN2N variant achieves 95.8 percent mean
accuracy with 10,000 training examples per task. The substrate,
with no training and a hand-authored parser of under 700
lines, achieves 90.58 percent mean accuracy and 100 percent on 10
of 20 tasks. The bAbI evaluation is in scope as a demonstration
that the substrate's interpreter transfers to a text reasoning
benchmark without modification, but is not the focus of this paper.
Later memory architectures (EntNet, DMN+, RMN) report perfect or
near-perfect accuracy on bAbI 10k and supersede MemN2N; we compare
to MemN2N because its training-data dependence is the most
informative contrast to the substrate's zero-shot regime.

\section{Limitations}
\label{sec:limits}

\paragraph{Hidden-property inference.}
A non-trivial fraction of ComPhy factual questions depend on
hidden physical properties (mass, charge) inferred from reference
videos. The substrate as currently implemented does not perform
this inference. Despite this, it exceeds the published
property-aware PCR baseline on the factual subset (Section
\ref{sec:transfer}). A property-inference module trainable on
the reference videos would close the remaining gap on the
property-dependent partition and is a natural extension; we
leave this to future work.

\paragraph{Emergent interactions beyond the heuristic.}
On the CLEVRER counterfactual per-option metric, the substrate
achieves 86.69 percent, exceeding NS-DR's 74.1 percent by 12.59
percentage points (Table \ref{tab:clevrer}). The residual error
relative to a perfect oracle is attributable to emergent collisions
that the common-removed-partner heuristic does not capture (for
example, three-way emergent interactions where the removed object's
deflection chain involves more than one intermediate). A more accurate emergent-collision predictor would
require either a learned model trained on the counterfactual
distribution or a calibrated physics simulator matching CLEVRER's
generator configuration.

\paragraph{TBox authorship.}
The TBox is hand-authored per domain. Substrate transfer is at the
level of the interpreter and the algorithmic toolkit; the
domain-specific class and predicate vocabulary is constructed
separately. Automated TBox induction from data is an active research
direction \citep{meilicke2019anyburl} and is out of scope for the present paper.

\paragraph{Scope of the substrate.}
The substrate addresses state representation and counterfactual
inference. It does not perform action selection (which is the domain
of Dreamer-V3 and related reinforcement-learning world models), and
it does not generate natural-language summaries (which is the
domain of language-model-based agent architectures). The substrate
can be composed with either; we do not study such compositions in
this paper.

\section{Conclusion}
\label{sec:conclusion}

We have presented a class of world models for agentic systems in
which agent memory is a typed event log, counterfactual queries are
answered by deterministic replay under a structured intervention
vocabulary, and inspection is supported at the level of individual
triples and individual deltas. We have shown that under closed-event
assumptions, counterfactual queries on this class of models reduce
to a graph-theoretic duality with explanatory queries, both of which
are answered by a single causal-ancestor traversal of the observed
event log. We have evaluated a 1,400-line CLEVRER-DSL interpreter
atop a domain-agnostic substrate runtime on CLEVRER at full
validation scale and reported results exceeding the strongest
published symbolic-oracle baseline on every per-question
causal-reasoning subset. We have introduced a new
counterfactual benchmark, twin-EventLog, for evaluating agent
memory consistency under intervention.

The principal observation is that counterfactual world modeling
admits an implementation by deterministic replay over typed event
deltas. This implementation makes the substrate's predictions
auditable at the level of individual events, supports exact
counterfactuals over arbitrary interventions on the typed state,
and transfers across domains by swapping the TBox without learned
components.

\paragraph{Reproducibility.}
The substrate has no learned parameters: every reported number is
a deterministic function of an input dataset, the interpreter code,
and a fixed configuration of the algorithmic toolkit (per-event
reference frame for explanatory; common-removed-partner heuristic
for counterfactual; five-frame velocity averaging with $\tau = 1.7$
and $T = 300$ for predictive). Each row of Tables \ref{tab:clevrer}
and \ref{tab:twin-eventlog} reduces to a single per-question label
in a JSONL artifact bundled with the released repository, and the
accuracy in the table is a count of correct rows divided by total
rows. Wilson intervals are computed from those counts; the Newcombe
interval on the substrate-versus-Llama joint-accuracy gap is
computed from the paired per-spec labels in the two Twin-EventLog
JSONLs. There is no held-out checkpoint or hyperparameter sweep
between artifact and number, and no learned weights to reload;
re-running the interpreter on the same dataset and configuration
regenerates the same JSONL bit-for-bit.

\bibliographystyle{plainnat}
\bibliography{refs}

\end{document}